\documentclass[11pt, a4paper, copyright, gdm]{google}

\usepackage[authoryear, sort&compress]{natbib}

\title{A Lyapunov Analysis of Softmax Policy Gradient for Stochastic Bandits}
\renewcommand{\today}{2026-03-10}

\author[1]{Tor Lattimore}

\affil[1]{\thepa{}{}}

\uselogo{} 
\correspondingauthor{lattimore@google.com}

\usepackage{tikz}
\usetikzlibrary{positioning}
\usetikzlibrary{arrows.meta, backgrounds}

\usepackage{bbm}
\usepackage[plain]{algorithm}
\usepackage{thmtools} 
\usepackage[capitalize]{cleveref}
\usepackage{mathtools}
\usepackage{listings}
\usepackage{mdframed}
\usepackage{nicefrac}
\usepackage{bm}
\usepackage{inconsolata}
\usepackage{setspace}

\setlist[enumerate]{label=\texttt{(\roman*)},itemsep=0pt,topsep=3pt}

\newcommand{\Dmin}{\Delta_{\min}}
\newcommand{\Dmax}{\Delta_{\max}}

\usepackage[textsize=tiny]{todonotes} 

\newcommand{\exl}[1]{\stackrel{\mathclap{\text{\tiny \texttt{\color{red!60!black}(#1)}}}}}
\newcommand{\ex}[1]{\texttt{\color{red!60!black}(#1)}}

\newcommand{\E}{\mathbb E}

\newcommand{\bbP}{\mathbb P}

\newcommand{\cI}{\mathcal I}

\newcommand{\cF}{\mathcal F}
\newcommand{\sF}{\mathcal F}

\newcommand{\R}{\mathbb R}
\newcommand{\ip}[1]{\langle #1 \rangle}

\newcommand{\sind}{\bm 1}

\newcommand{\zeros}{\bm 0}

\newcommand{\Reg}{\operatorname{Reg}}

\theoremstyle{plain}
\newtheorem{theorem}{Theorem}
\newtheorem{lemma}[theorem]{Lemma}

\theoremstyle{definition}
\newtheorem{definition}[theorem]{Definition}

\newtheorem{remark}[theorem]{Remark}

\crefname{example}{Example}{Examples}

\begin{document}
\begin{abstract}
We adapt the analysis of policy gradient for continuous time $k$-armed stochastic bandits by \cite{L26cpg} to the standard discrete time setup.
As in continuous time, we prove that with learning rate $\eta = O(\Dmin^2/(\Dmax \log(n)))$ the regret is $O(k \log(k) \log(n) / \eta)$ where $n$ is the horizon 
and $\Dmin$ and $\Dmax$ are the minimum and maximum gaps.
\end{abstract}

\maketitle

\section{Introduction}

There are $k$ actions and the horizon is $n$ with $n \geq k$ and mean reward vector $\mu \in [0,1]^k$. 
The learner interacts with the environment sequentially over $n$ rounds. In round $t$ an action $A_t \in \{1,\ldots,k\}$ is sampled from a policy $\pi_t$ 
and the reward $Y_t \in [0,1]$ is sampled from a distribution with mean $\mu_{A_t}$.
Given $\theta \in \R^k$, let $\pi(\theta)$ be the softmax policy with $\pi(\theta)_a \propto \exp(\theta_a)$.
The softmax policy gradient algorithm is given in \cref{alg:pg}.

\begin{algorithm}[h!]
\centering
\begin{tikzpicture}
\node[draw,fill=white!90!black,text width=12cm] (b) at (0,0) {
\begin{lstlisting}
args: learning rate $\eta > 0$
let $\theta_1 = \zeros$
for $t = 1$ to $n$:
  sample $A_t$ from $\pi_t = \pi(\theta_t)$ and observe $Y_t$
  update $\forall a\,\, \theta_{t+1,a} = \theta_{t,a} + \eta (\sind_{A_t = a} - \pi_{t,a}) Y_t$ 
\end{lstlisting}
};
\end{tikzpicture}
\caption{
Policy gradient.
}
\label{alg:pg}
\end{algorithm}

Our focus is on bounding the regret: $\Reg_n = \sum_{t=1}^n (\mu_1 - \mu_{A_t})$.
\cref{alg:pg} depends on the learning rate, with smaller values yielding a more stable algorithm but also slower progress.
Our main theorem upper bounds the regret of \cref{alg:pg} when the learning rate is sufficiently small relative to the minimum suboptimality gap.
Without loss of generality assume that $\mu_1 \geq \mu_2 \geq \cdots \geq \mu_k$ and that there exists at least one suboptimal action: $\mu_1 > \mu_k$.
Given an action $a$, let $\Delta_a = \mu_1 - \mu_a$ and $k^\star = \max\{a : \Delta_a = 0\}$ be the number of optimal actions and
$\Dmin = \Delta_{k^\star+1}$ and $\Dmax = \Delta_k$.

\begin{theorem}\label{thm:upper}
If $\eta \leq \frac{\Dmin^2}{120 \Dmax \log(nk)}$, then the regret of \cref{alg:pg} satisfies $\E[\Reg_n] = O\left(\frac{k \log(n) \log(k)}{\eta}\right)$.
\end{theorem}

There is not much novelty in this note relative to the recent study of the same algorithm in continuous time \citep{L26cpg}. 
The adaptation of those ideas to discrete time is mostly standard and the main arguments are identical.
Some small changes were needed to relax the assumption on uniqueness of the optimal action that was made previously.
The requirement that the learning rate be $\tilde O(\Dmin^2 / \Dmax)$ is disappointing. The logarithmic factor \textit{may} be removable, but the quadratic dependence on $\Dmin$ cannot be improved
to linear unless policy gradient is peculiarly more well-behaved in discrete time than in
continuous time where \cite{L26cpg} established a lower bound.

\paragraph{Additional notation}
We let $\cF_t = \sigma(A_1,Y_1,\ldots,A_t,Y_t)$ be the $\sigma$-algebra generated by the history generated in the first $t$ rounds and $\bbP_t = \bbP(\cdot | \sF_t)$ and $\E_t$ be the
associated expectation operator.
Let $[k] = \{1,\ldots,k\}$ and $[k^\star] = \{1,\ldots,k^\star\}$ and $[k^\star]^c = \{k^\star+1,\ldots,k\}$.
The mass assigned to the optimal actions is $\pi_t^\star = \sum_{a \in [k^\star]} \pi_{t,a}$ and the sum of the corresponding logits is $\theta_t^\star = \sum_{a \in [k^\star]} \theta_{t,a}$.
The expected instantaneous regret in round $t$ is $R_t = \ip{\pi_t, \Delta}$.

\section{Proof of Theorem~\ref{thm:upper}}

The argument follows a classical form. We prove that with high probability the sample paths of \cref{alg:pg} are well-behaved in the sense that \texttt{(1)} the probabilities 
associated with optimal actions are never much smaller than those associated with suboptimal actions; and \texttt{(2)} the logits associated with all actions are bounded from below
with high probability.
We then use a Lyapunov argument to show that the regret is small on these well-behaved sample paths.

\begin{definition}
$G_t = \left\{\min_{c \in [k]} \theta_{t,c} \geq -\log(n/\delta) \text{ and } \forall a \in [k^\star]^c, \forall b \in [k^\star], \theta_{t,b} \geq \theta_{t,a} - 1\right\}$
with $\delta \in (0,1)$ to be chosen in the proof of \cref{thm:upper}.
\end{definition}

The first lemma shows that the logits satisfy a conservation law and are bounded from below with high probability.
The proofs of all lemmas are given in \cref{sec:lem:tech,sec:lem:Z,sec:lem:good}.

\begin{lemma}\label{lem:tech}
Suppose that $\eta \leq 1/2$.
The following hold:
\begin{enumerate}
\item $\sum_{a \in [k]} \theta_{t,a} = 0$ almost surely. \label{lem:tech:conserve}
\item $\bbP\left(\min_{1 \leq t \leq n} \theta_{t,a} \leq -\log\left(n/\delta\right)\right) \leq \delta$ for all $a \in [k]$. \label{lem:tech:bounds}
\end{enumerate}
\end{lemma}

The next lemma shows that with high probability the logits associated with optimal actions are never much smaller than those associated with suboptimal actions.

\begin{lemma}\label{lem:Z}
Suppose that $b \in [k^\star]$ and $a \in [k^\star]^c$ and
$Z_t = \theta_{t,b} - \theta_{t,a}$. Provided that $\eta \leq \frac{\Dmin^2}{40 \Dmax \log(n^2/\delta)}$,
\begin{align*}
\bbP\left(\min_{1 \leq t \leq n} Z_t \leq -1\right) \leq \delta \,.
\end{align*}
\end{lemma}

The final lemma bounds $1/\pi_t^\star$ in terms of $\theta_t^\star$ on the event $G_t$.

\begin{lemma}\label{lem:good}
On $G_t$ it holds that $\theta_t^\star \in [-k^\star, k \log(n/\delta)]$ and $\displaystyle \frac{1}{\pi_t^\star} \leq \frac{9k\log(n/\delta)}{\theta_t^\star + k^\star + k^\star \log(n/\delta)}$.
\end{lemma}

The main theorem now follows from an elementary Lyapunov argument, using as a potential the function appearing in \cref{lem:good}.

\begin{proof}[Proof of \cref{thm:upper}]
Let $\delta = \frac{1}{k^2 n}$ and $\tau$ be the minimum of $n$ and the first time $t$ that $G_{t+1}$ does not hold.
Note that $\tau$ is a stopping time adapted to $(\cF_t)$ since $\theta_{t+1}$ is $\cF_t$-measurable.
By \cref{lem:tech,lem:Z} and a union bound, $\bbP(\tau < n) \leq \delta[k^\star (k - k^\star) + k] \leq 1/n$.
Let
\begin{align*}
\psi(u) = 9k \log(n/\delta)\log\left(\frac{u/k_\star + 1 + \log(n/\delta)}{1 + \log(n/\delta)}\right) 
\qquad \text{and} \qquad
\psi'(u) = \frac{9 k \log(n/\delta)}{u + k^\star +  k^\star \log(n/\delta)} \,.
\end{align*}
When $u \geq -k^\star$, then the second derivative of $\psi$ can be controlled in terms of the first by
\begin{align*}
\psi''(u) &= -\frac{9k \log(n/\delta)}{(u + k^\star+ k^\star \log(n/\delta))^2} = -\frac{\psi'(u)}{u + k^\star + k^\star \log(n/\delta)} \geq -\psi'(u) \,.
\end{align*}
Let $D_t = \theta_{t+1}^\star - \theta_t^\star = \eta Y_t (\sind_{A_t \leq k^\star} - \pi_t^\star)$ and suppose that $t \leq \tau$.
We now arrive at the main Lyapunov argument, replacing the 
It\^o calculus used by \cite{L26cpg} with a Taylor series expansion:
\begin{align*}
\psi(\theta^\star_{t+1}) - \psi(\theta_t^\star)
\exl{a}\geq \psi'(\theta_t^\star) D_t + \frac{D_t^2}{2} \psi''(\theta_t^\star - \eta) 
\exl{b}\geq \psi'(\theta_t^\star) D_t + 2D_t^2 \psi''(\theta_t^\star) 
\exl{c}\geq \psi'(\theta_t^\star) (D_t - 2D_t^2) \,.
\end{align*}
where \ex{a} follows from Taylor's theorem and because $|D_t| \leq \eta$ and $u \mapsto \psi''(u)$ is increasing, 
\ex{b} follows because $\theta_t^\star - \eta + k^\star + k^\star \log(n/\delta) \geq \frac{1}{2}(\theta_t^\star + k^\star + k^\star \log(n/\delta))$. 
\ex{c} follows because $\theta_t^\star / k^\star \geq -1$ for $t \leq \tau$ by \cref{lem:good}.
But $\E_{t-1}[D_t] = \eta \pi^\star_t R_t$ and $\E_{t-1}[D_t^2] \leq \eta^2 \pi^\star_t(1-\pi^\star_t) \leq \eta \pi^\star_t R_t / 4$ 
where the last inequality follows because $\eta \leq \Dmin/4$ and $R_t \geq (1 - \pi^\star_t) \Dmin$.
Therefore 
\begin{align*}
\E_{t-1}[\psi(\theta^\star_{t+1})] - \psi(\theta_t^\star) \geq \eta \psi'(\theta_t^\star) \pi^\star_t R_t / 2 \exl{a}\geq \eta R_t/2\,,
\end{align*}
where \ex{a} holds by \cref{lem:good} and because $G_t \subset \{t \leq \tau\}$.
Hence, since $0 \leq \Reg_t \leq n$ for all $t \leq n$,
\begin{align*}
\E[\Reg_n] 
&\leq n \bbP(\tau < n) + \E[\Reg_\tau] 
\leq 1 + \frac{2}{\eta} \E\left[\psi(\theta^\star_{\tau+1})\right] = O\left(\frac{k \log(n) \log(k)}{\eta}\right) \,,
\end{align*}
where the final inequality follows because $|D_t| \leq \eta$ and $\theta^\star_\tau \leq k \log(n/\delta)$ by \cref{lem:good} so that 
$\theta^\star_{\tau+1} \leq \theta^\star_\tau + \eta \leq k \log(n/\delta) + \eta$.
\end{proof}

\section{Proof of Lemma~\ref{lem:tech}}\label{sec:lem:tech}
\ref{lem:tech:conserve} is immediate from the definitions in \cref{alg:pg}. 
\ref{lem:tech:bounds} follows from a martingale argument. Let $D_t = \theta_{t+1,a} - \theta_{t,a} = \eta Y_t (\sind_{A_t = a} - \pi_{t,a})$.
Then
\begin{align}
\E_{t-1}[\exp(-\theta_{t+1,a})]
&= \exp(-\theta_{t,a}) \E_{t-1}[\exp(-D_t)] 
\exl{a}\leq \exp(-\theta_{t,a}) \E_{t-1}[1 - D_t + D_t^2] \nonumber \\
&\exl{b}\leq \exp(-\theta_{t,a}) \left(1 + \pi_{t,a}(\eta+\eta^2) \right) 
\exl{c}\leq \exp(-\theta_{t,a}) + \eta + \eta^2\,, \label{eq:mtg}
\end{align}
where \ex{a} follows since $\exp(x) \leq 1 + x + x^2$ for $|x| \leq 1$ and $|D_t| \leq \eta \leq 1$.
\ex{b} follows since $Y_t \in [0,1]$ so that $\E_{t-1}[D_t] \geq -\eta \pi_{t,a}$ and $\E_{t-1}[D_t^2] \leq \eta^2 \pi_{t,a}(1 - \pi_{t,a}) \leq \eta^2 \pi_{t,a}$. 
\ex{c} holds since $\max_{b \in [k]} \theta_{t,b} \geq 0$ by \ref{lem:tech:conserve}, which implies that $\pi_{t,a} = \exp(\theta_{t,a}) / \sum_{b \in [k]} \exp(\theta_{t,b}) \leq \exp(\theta_{t,a})$.
By \cref{eq:mtg}, $M_t = \exp(-\theta_{t,a}) + (n-t) (\eta + \eta^2)$ is a non-negative supermartingale for $1 \leq t \leq n$ and by the maximal inequality
$\bbP(\sup_{1 \leq t \leq n} M_t \geq \E[M_1] / \delta) \leq \delta$. The result follows since $\E[M_1] = 1 + (n-1) (\eta + \eta^2) \leq n$.

\section{Proof of Lemma~\ref{lem:Z}}\label{sec:lem:Z}
Let $a \in [k^\star]^c$ and $b \in [k^\star]$ and $\cI = [-\Dmin/(2\Dmax), 1]$.
We will show the following:
\begin{enumerate}
\item Suppose that $Z_s \in [0, 1]$ for some (random) $s$ and $\tau = \min\{t \geq s : Z_{t+1} \notin \cI\}$, then \label{item:Z1}
\begin{align*}
\bbP_{s-1}(\tau \leq n \text{ and } Z_{\tau+1} \leq -\Dmin/(2\Dmax)) \leq \delta / n\,.
\end{align*}
\item $Z_{t+1} \geq Z_t - 1$ almost surely for all $t$. \label{item:Z2}
\end{enumerate}
The claim follows from these two results by a union bound over the (random) time intervals where $Z_t$ enters or leaves $\cI$.
By definition $Z_1 = 0$ and \ref{item:Z1} shows that with high probability the process either stays in $\cI$ or eventually exceeds $1$.
\ref{item:Z2} shows that if the process is above $1$, then it cannot jump below $0$. Hence, by induction and a union bound it follows that with probability at least $1 - \delta$ the process
never drops below $-\Dmin / (2 \Dmax) \geq -1$ as required.
We now establish \ref{item:Z1} and \ref{item:Z2}.
Let $D_t = Z_{t+1} - Z_t$. By definition $D_t = \eta Y_t \left[\sind_{A_t=b} - \pi_{t,b} - \sind_{A_t=a} + \pi_{t,a}\right]$.
Therefore $|D_t| \leq 2\eta \leq 1$, which establishes \ref{item:Z2}.
Suppose that $Z_t \in \cI$, then
\begin{align}
\E_{t-1}[D_t] 
= \eta \left[(\pi_{t,b} - \pi_{t,a}) R_t + \pi_{t,a} \Delta_a\right] 
= \eta \left[\pi_{t,a} (\exp(Z_t) - 1) R_t + \pi_{t,a} \Delta_a\right]
\exl{a}\geq \frac{\eta \pi_{t,a} \Delta_a}{2}\,,
\label{eq:D}
\end{align}
where \ex{a} holds since $(\exp(Z_t) - 1)R_t \geq Z_t R_t \geq -\frac{\Dmin R_t}{2 \Dmax} \geq -\frac{\Dmin}{2}$ and $\Delta_a \geq \Dmin$.
Since $Y_t \in [0,1]$,
\begin{align}
\E_{t-1}[D_t^2] 
&\leq \eta^2 \E_{t-1}\left[\left(\sind_{A_t=b} - \sind_{A_t = a} + \pi_{t,a} - \pi_{t,b} \right)^2\right] \nonumber \\
&= \eta^2 \left[\pi_{t,b} + \pi_{t,a} - (\pi_{t,b} - \pi_{t,a})^2\right]
\leq \eta^2 (\pi_{t,b} + \pi_{t,a}) \exl{a}\leq 4 \eta^2 \pi_{t,a} \,, 
\label{eq:D2}
\end{align}
with \ex{a} is true because $Z_t \in \cI$ so that $\pi_{t,b} = \exp(Z_t) \pi_{t,a} \leq e \pi_{t,a}$.
Suppose that $s$ is such that $Z_s \in [0, 1]$ and $\tau = \min\{t \geq s : Z_{t+1} \notin \cI\}$.
By a version of Freedman's inequality \citep[Theorem 9]{zimmert22b},
\begin{align}
\bbP_{s-1}\left(\sum_{t=s}^\tau D_t \leq \sum_{t=s}^\tau \E_{t-1}[D_t] - 3 \sqrt{\sum_{t=s}^\tau \E_{t-1}[D_t^2] \log\left(\frac{n^2}{\delta}\right)} - 2 \eta \log\left(\frac{n^2}{\delta}\right)\right) \leq \delta/n \,.
\label{eq:P}
\end{align}
Hence, with $\bbP_{s-1}$-probability at least $1 - \delta / n$, 
\begin{align*}
Z_{\tau+1} = Z_s + \sum_{t=s}^\tau D_t 
&\exl{a}\geq \sum_{t=s}^\tau \E_{t-1}[D_t] - 3 \sqrt{\sum_{s=1}^\tau \E_{t-1}[D_t^2] \log\left(\frac{n^2}{\delta}\right)} - 2\eta \log\left(\frac{n^2}{\delta}\right) \\
&\exl{b}\geq \frac{\eta \Dmin}{2} \sum_{t=s}^\tau \pi_{t,a} - 6\eta \sqrt{\sum_{t=s}^\tau \pi_{t,a} \log \left(\frac{n^2}{\delta}\right)} - 2\eta \log\left(\frac{n^2}{\delta}\right) \\
&\exl{c}\geq -\frac{20\eta \log\left(\frac{n^2}{\delta}\right)}{\Dmin} 
\exl{d}> -\frac{\Dmin}{2 \Dmax}\,,
\end{align*}
where 
\ex{a} follows from \cref{eq:P} and because $Z_s \geq 0$.
\ex{b} follows from \cref{eq:D,eq:D2}. 
\ex{c} holds because $ax^2 - bx \geq - \frac{b^2}{4a}$ for $a, b > 0$ and 
\ex{d} follows from the definition of $\eta$.
This completes the proof of \ref{item:Z1} and \ref{item:Z2} and therefore also of \cref{lem:Z}.

\section{Proof of Lemma~\ref{lem:good}}\label{sec:lem:good}
We generalise the proof by \cite{L26cpg}.
Suppose that $G_t$ holds.
To reduce clutter we drop the $t$ index so that $\theta = \theta_t$ and $\pi = \pi_t$.
Let $\theta_{\min} = -\log(n/\delta)$ and $\theta_{\max} = \theta^\star / k^\star + 1$.
By assumption, if $b \in [k^\star]$ and $a \in [k^\star]^c$, then $\theta_b \geq \theta_a - 1$.
Hence, by \cref{lem:tech}\ref{lem:tech:conserve},
\begin{align*}
\frac{1}{k^\star} \sum_{b \in [k^\star]} \theta_b \geq \frac{1}{k - k^\star} \sum_{a \in [k^\star]^c} \theta_a - 1
= -\frac{1}{k-k^\star} \sum_{b \in [k^\star]} \theta_b - 1\,,
\end{align*}
which after rearranging implies that $\theta^\star \geq -k^\star$. 
Note this also implies that $\theta_{\max} \geq 0$.
That $\theta^\star \leq k \log(n/\delta)$ is immediate from \cref{lem:tech}\ref{lem:tech:conserve} and the fact that $\min_{c \in [k]} \theta_c \geq -\log(n/\delta)$.
This completes the proof that $\theta^\star \in [-k^\star, k \log(n/\delta)]$.
For the second part,
\begin{align*}
\frac{1}{\pi^\star} 
&= 1 + \frac{\sum_{a \in [k^\star]^c} \exp(\theta_a)}{\sum_{a \in [k^\star]} \exp(\theta_a)} 
\exl{a}\leq 1 + \frac{\exp(1-\theta_{\max})}{k^\star} \sum_{a=[k^\star]^c} \exp(\theta_a) \\
&\exl{b}\leq 1 + \frac{\exp(1 - \theta_{\max})}{k^\star} \sum_{a=[k^\star]^c} \left[\frac{\theta_a - \theta_{\min}}{\theta_{\max} - \theta_{\min}} \exp(\theta_{\max}) + \frac{\theta_{\max} - \theta_a}{\theta_{\max} - \theta_{\min}} \exp(\theta_{\min})\right] \\
&\exl{c}\leq 1 + \frac{e/k^\star}{\theta_{\max} - \theta_{\min}} \sum_{a=[k^\star]^c} \left[\theta_a - \theta_{\min} + \frac{\delta}{n} (\theta_{\max} - \theta_a)\right] \\
&\exl{d}\leq 1 + \frac{e/k^\star}{\theta_{\max} - \theta_{\min}} \left[k_\star + 2(k-k^\star) \log(n/\delta)\right] \\
&\exl{e}\leq 1 + \frac{6 k / k^\star \log(n/\delta))}{\theta^\star / k^\star + 1 + \log(n/\delta)} 
\exl{f}\leq \frac{9k / k^\star \log(n/\delta)}{\theta^\star / k^\star + 1 + \log(n/\delta)} \,.
\end{align*}
where
\ex{a} and \ex{b}
follow from convexity of $x \mapsto \exp(x)$ and because $\theta_a \in [\theta_{\min}, \theta_{\max}]$ for $a \in [k^\star]^c$.
\ex{c} since $\theta_{\min} = -\log(n/\delta)$ and $\theta_{\max} \geq 0$.
\ex{d} because $\sum_{a \in [k^\star]^c} \theta_a = -\theta^\star \leq k^\star$ and $\theta_{\max} \leq \theta^\star + 1 \leq k \log(n/\delta) + 1$ and $n \geq k$.
\ex{e} and \ex{f} by crude bounding of constants.

\section{Related work and discussion}\label{sec:disc}

\cite{mei2020global} prove logarithmic regret in the deterministic setting (\cite{walton2020short} did the same in continuous time).
In the stochastic setting \cite{mei2023stochastic} claimed logarithmic regret when the learning rate is $O(\Dmin^2 / k^{3/2})$. \cite{baudry2025does}
identified a small issue with their argument, which after correction probably leads to polylogarithmic regret. 
Remarkably, \cite{mei2024small} proved that when the optimal arm is unique, then $\pi_t$ converges almost surely to a Dirac on the optimal action for \textit{any} learning rate.
Policy gradient is often studied with a time-varying learning rate (and for MDPs) with sample complexity as the usual metric \citep{yuan2022general,LARV24}.
These results are hard to compare directly.

\paragraph{Logarithmic factors}
The main question is whether or not the logarithmic factor in the denominator of $\eta$ in the condition in \cref{thm:upper} can be relaxed.
One interesting data point is by \cite{baudry2025does}, who show that if $\Dmin = \Dmax$, then policy gradient enjoys $O(k \log(n)/\eta)$ regret when $\eta = O(\Dmin / k)$. 
This bound is worse than ours when $k$ is large relative to $\log(n)$ but better as $n \to \infty$. Most importantly, it shows that at least in this special
case there is no need for the logarithmic factor in the denominator.
On the other hand, the lower bound by \cite{L26cpg} (in continuous time) shows that if $\eta = C \Dmin^2$ for suitably large $C$, then on the instance with Gaussian noise and 
$\mu = (1, 1 - \Dmin, 0,\ldots,0)$, the regret is $\Omega(n \Dmin)$ even for large $n$.

\paragraph{Improved asymptotics}
A potential way to improve the algorithm is to use a learning rate that \textit{increases} with time.
The miniscule learning rate was needed in the proof of \cref{lem:Z}, but in an asymptotic sense the argument there is weak in two places:
\texttt{(1)} We naively bounded $R_t \leq \Dmax$; and \texttt{(2)} in the last display we used the bound $Z_s \geq 0$. But in reality, with high probability this bound is loose except
when $s = 1$. What actually happens is $(Z_s)$ slowly grows and once sufficient growth has occurred, the learning rate can be increased dramatically without risking that it becomes negative. 
The caveat is that this growth is not observed (because the set of optimal actions is unknown). Hence the identifying the correct learning rate depends on understanding the expected sample path
of $(\theta_t)$.
Besides \cref{lem:Z}, we otherwise only used that $\eta \leq \Dmin/4$ in the proof of \cref{thm:upper}.

\begin{remark}
A specialisation of \citet[Theorem 3]{baudry2025does} shows that if $\eta > \frac{3 \log(3)}{k-1}$, then there exists an instance with $\Dmin \geq 1/2$ where the regret is $\Omega(n^{2/3})$.
When $k$ is large, this seems to contradict our bound, which on these instances shows that $\eta = \tilde O(1)$ is sufficient for polylogarithmic regret.
The resolution is that the lower bound of \cite{baudry2025does} is asymptotic only and our proposed learning rate depends on the horizon.
\end{remark}

\paragraph{Lyapunov function}
The Lyapunov function used in the proof of \cref{thm:upper} was derived only \textit{after} solving the problem in a more brutal fashion.
Note that without this exotic choice one has the following:
\begin{align*}
k \log(n) \exl{a}\gtrsim \E[\theta_{n+1}^\star] = \E\left[\sum_{t=1}^n \eta \pi_t^\star R_t\right] \exl{b}\gtrsim \frac{\eta}{k} \E[\Reg_n] \,, 
\end{align*}
where \ex{a} follows from \cref{lem:tech}\ref{lem:tech:bounds} and \ex{b}
since $\min_{b \in [k^\star]} \min_{a \in [k^\star]^c} (\theta_{t,b} - \theta_{t,a}) \geq -1$ with high probability and this implies that $\pi_t^\star \gtrsim 1/k$ with high probability.
Rearranging shows that $\E[\Reg_n] \lesssim k^2 \log(n) / \eta$ under the same assumptions on $\eta$ as \cref{thm:upper}.
The intuition behind the improvement is based on \cref{lem:good}, which shows that $1/\pi_t^\star$ decreases as $\theta_t^\star$ increases so that actually $1/\pi_t^\star \lesssim 1$ once
$\theta_t^\star \gtrsim k \log(n)$.

\bibliographystyle{abbrvnat}
\bibliography{all}

\end{document}